\documentclass{article}

\usepackage{microtype}
\usepackage{graphicx}
\usepackage{subcaption}
\usepackage{booktabs} % for professional tables
\usepackage{algorithm}

\usepackage{hyperref}

\usepackage[preprint]{icml2026}

\usepackage{amsmath}
\usepackage{amssymb}
\usepackage{mathtools}
\usepackage{amsthm}
\usepackage{multirow} 
\usepackage{booktabs}

\usepackage[capitalize,noabbrev]{cleveref}

\theoremstyle{plain}
\newtheorem{theorem}{Theorem}[section]
\newtheorem{proposition}[theorem]{Proposition}

\theoremstyle{definition}

\newtheorem{assumption}[theorem]{Assumption}
\theoremstyle{remark}
\newtheorem{remark}[theorem]{Remark}

\usepackage{amsfonts,amsmath,amssymb,stmaryrd,bbm, amsthm }
\newcommand{\bfa}{\mathbf{a}}  
  
\newcommand{\bfg}{\mathbf{g}}

 \newcommand{\bfw}{\mathbf{w}}

\newcommand{\bfzero}{\mathbf{0}} 

\renewcommand{\P}{\mathcal{P}}  
  \newcommand{\Uc}{\mathcal{U}}

 \newcommand{\EE}{\mathbb{E}}

\newcommand{\PP}{\mathbb{P}}  \newcommand{\RR}{\mathbb{R}}

\renewcommand{\d}{\mathrm{d}} % for the integrals 
\newcommand{\argmin}{{\mathrm{argmin}}}

\newcommand{\OT}{{\mathrm{OT}}}

\newcommand{\hatrhoa}{\hat{\rho}_{m,\mathbf{a}}}

\newcommand{\hatrho}{\hat{\rho}_{n}}

\newcommand{\e}{\varepsilon}
\newcommand{\dint}{d_{\mathrm{int}}}
\renewcommand{\hat}{\widehat}

\usepackage[textsize=tiny]{todonotes}
\icmltitlerunning{Geometry-Aware Optimal Transport: Fast Intrinsic Dimension and Wasserstein Distance Estimation}

\begin{document}

\onecolumn
  \icmltitle{Geometry-Aware Optimal Transport: Fast Intrinsic Dimension and Wasserstein Distance Estimation}

  \icmlsetsymbol{equal}{*}

  \begin{icmlauthorlist}
    \icmlauthor{Ferdinand Genans}{lpsm}
    \icmlauthor{Olivier Wintenberger}{lpsm,wpi}
  \end{icmlauthorlist}

  \icmlaffiliation{lpsm}{LPSM, Sorbonne Université, Paris, France}
  \icmlaffiliation{wpi}{Wolfgang Pauli Institute, Vienna, Austria}

  \icmlcorrespondingauthor{Ferdinand Genans}{genans.ferdinand@gmail.com}

  % You may provide any keywords that you find helpful for describing your
  % paper; these are used to populate the "keywords" metadata in the PDF but
  % will not be shown in the document
  \icmlkeywords{Machine Learning, ICML}

  \vskip 0.3in

% this must go after the closing bracket ] following \twocolumn[ ...

% This command actually creates the footnote in the first column listing the
% affiliations and the copyright notice. The command takes one argument, which
% is text to display at the start of the footnote. The \icmlEqualContribution
% command is standard text for equal contribution. Remove it (just {}) if you
% do not need this facility.

% Use ONE of the following lines. DO NOT remove the command.
% If you have no special notice, KEEP empty braces:
\printAffiliationsAndNotice{}  % no special notice (required even if empty)
% Or, if applicable, use the standard equal contribution text:
% \printAffiliationsAndNotice{\icmlEqualContribution}
\begin{abstract}
Solving large scale Optimal Transport (OT) in machine learning typically relies on sampling measures to obtain a tractable discrete problem. While the discrete solver's accuracy is controllable, the rate of convergence of the discretization error is governed by the intrinsic dimension of our data. Therefore, the true bottleneck is the knowledge and control of the sampling error. In this work, we tackle this issue by introducing novel estimators for both sampling error and intrinsic dimension. The key finding is a simple, tuning-free estimator of $\text{OT}_c(\rho, \hat\rho)$ that utilizes the semi-dual OT functional and, remarkably, requires no OT solver. Furthermore, we derive a fast intrinsic dimension estimator from the multi-scale decay of our sampling error estimator. This framework unlocks significant computational and statistical advantages in practice, enabling us to (i) quantify the convergence rate of the discretization error, (ii) calibrate the entropic regularization of Sinkhorn divergences to the data's intrinsic geometry, and (iii) introduce a novel, intrinsic-dimension-based Richardson extrapolation estimator that strongly debiases Wasserstein distance estimation. Numerical experiments demonstrate that our geometry-aware pipeline effectively mitigates the discretization error bottleneck while maintaining computational efficiency.
\end{abstract}

\section{Introduction}

Many modern machine learning objects are most naturally modeled as probability distributions: think images (as distributions of pixel intensities), point clouds, word embeddings, or empirical datasets viewed through their sampling measures. A principled way to compare and manipulate such objects is provided by Optimal Transport (OT), which seeks the lowest-cost plan for moving mass from one distribution to another. Through this formulation, one obtains the family of Wasserstein (Earth Mover’s) distances and associated transport plans, which have become effective building blocks in complex ML pipelines \cite{peyre2019computational}, including generative modeling \cite{an2019ae, chen2019gradual}, speeding up diffusion models \cite{li2023dpm}, domain adaptation \cite{courty2014domain}, and natural language processing \cite{kusner2015word}.

These applications typically operate in the large-scale regime: data may arise from continuous distributions or from discrete distributions supported on an intractably large number of points. In practice, one therefore works with discretized approximations obtained by sampling. This simple step, which replaces unknown measures with empirical ones, has spurred two complementary research threads: (i) Statistical OT, which quantifies the error induced by this discretization; and (ii) Computational OT, which designs solvers that scale to the resulting large discrete problems.

\textbf{Statistical OT in theory.} A central question is how well OT quantities computed on samples approximate their population counterparts. The line of work initiated by \cite{dudley1969speed} analyzed the 1-Wasserstein discrepancy $W_1(\rho_n,\rho)$ when a continuous measure $\rho$ is replaced by its empirical counterpart $\rho_n$ built from $n$ i.i.d. samples with uniform weights. This was later generalized in \cite{fournier2015rate}, showing that $W_p(\rho_n,\rho)$ typically scales as $\mathcal{O}(n^{-1/d})$ in $\mathbb{R}^d$ for all $p\ge 1$, making explicit the curse of dimensionality: the number of samples required to accurately approximate OT quantities grows exponentially with dimension. Estimating the transport \emph{map} itself also suffers from this curse, as established in \cite{hutter2021minimax}. Several strategies mitigate these effects: exploiting low intrinsic dimensional structure \cite{WeedSharpRates}, or focusing on special settings such as semi-discrete OT \cite{pooladian2023minimax}. We refer to the recent book \cite{chewi2024statistical} for a comprehensive overview.

\textbf{OT and Intrinsic Dimension.} A key insight mitigating the curse of dimensionality is that real-world data rarely fully occupy the ambient space $\mathbb{R}^d$. As established in \cite{WeedSharpRates}, if the measure $\rho$ is supported on a manifold $\mathcal{M}$ of lower intrinsic dimension $d_{\mathcal{M}} \ll d$, the estimation error improves significantly to $\mathcal{O}(n^{-1/d_{\mathcal{M}}})$. Furthermore, OT exhibits a favorable multiscale behavior: the approximation error is not governed by a single global dimension, but rather tracks the covering numbers of the support at varying radii, often yielding superior non-asymptotic performance. Crucially, when  using OT to compare two different measures, this complexity adapts to the \textit{minimum} intrinsic dimension of the two measures \cite{hundrieser2024empirical}, and these properties extend to the entropic setting \cite{stromme2024minimum}.

\textbf{Computational OT.} On the algorithmic side, discrete OT can be posed as a linear program, but generic LP solvers are prohibitive at scale. The practical impact of OT in ML stems from specialized, structure-exploiting algorithms. A turning point was \cite{cuturi2013sinkhorn}, which introduced entropically regularized OT (EOT) solved efficiently by Sinkhorn iterations with complexity $\mathcal{O}(n^2/\varepsilon^2)$ to approximate discrete OT with $\varepsilon$-accuracy. Subsequent work improved constants and convergence guarantees, including primal-dual methods \cite{dvurechensky2018computational, lin2019efficient} and recent advances achieving $\mathcal{O}(n^2/\varepsilon)$ \cite{blanchet2023towards}, which appears optimal for this class of problems.

\textbf{The Sinkhorn Divergence: statistical-computational bridge and its limits.} A principled attempt to bridge statistical and computational OT lies in the study of \textit{Sinkhorn Divergences} \cite{ramdas2017wasserstein, genevay2018learning, feydy2019interpolating}. By explicitly defining the discrepancy through the discrete Sinkhorn divergence, one achieves a parametric estimation rate of $\mathcal{O}(\varepsilon^{-d/2}n^{-1/2})$  \cite{genevay2019sample, mena2019statistical, chizat2020faster}. Consequently, the curse of dimensionality is not removed but rather shifted into the regularization parameter $\varepsilon$. Crucially, the exponent $d$ governing this trade-off is the \textit{intrinsic} dimension $d_{\mathcal{M}}$ of the data manifold \cite{stromme2024minimum}, and a well-calibrated speed at which $\varepsilon \to 0$ nearly achieves the minimax rates to estimate the Wasserstein distance.  

\textbf{Missing Pieces.} Despite this progress, a fully operational link between statistical theory and computational practice remains elusive. While Sinkhorn Divergences offer a theoretical trade-off, calibrating it well requires knowledge of the typically unknown underlying intrinsic dimension $d_{\mathcal{M}}$ to select the optimal $\varepsilon$. As a result, practitioners are left without concrete guidance: they often drive numerical solvers to very high precision, implicitly treating the sampling error as negligible, or tune $\varepsilon$ heuristically without addressing the bias-variance trade-off. To apply OT with calibrated, problem-dependent accuracy, one requires a practical way to estimate the discretization error and the intrinsic dimension directly from the available samples. 

\paragraph{Contributions.}
To address these challenges, we propose a geometry-aware framework that estimates both the discretization error and the intrinsic dimension directly from samples, without relying on expensive ground-truth computations. Our specific contributions are:

\begin{itemize}
    \item \textbf{A solver-free discretization error estimator.} We introduce a novel estimator for the quantization error $\text{OT}_c(\rho,\rho_n^*)$, where $\rho_n^*$ represents the empirical measure with weights optimally adjusted to minimize the transport cost. Our key theoretical insight is that the semi-dual formulation of this problem admits a closed-form solution. This allows us to estimate the error via simple Monte Carlo integration, bypassing the need for an OT solver entirely. The resulting estimator is parameter-free, highly parallelizable on GPUs, and computationally inexpensive.

    \item \textbf{A scalable intrinsic dimension estimator.} Leveraging the convergence behavior of our error estimator, we derive an estimator of the empirical intrinsic dimension $d_{\text{int}}$ of the data. By analyzing the decay of the error as a function of sample size, our estimator captures the multi-scale geometric structure of the distribution, consistent with the theoretical findings in \cite{WeedSharpRates}. Crucially, this method scales linearly (i.e., $\mathcal{O}(n)$) in both time and space $\mathcal{O}(n)$, making it applicable to large-scale datasets.

    \item \textbf{Debiased Wasserstein estimation via Diagonal Richardson Extrapolation.} We propose a new extrapolation scheme that combines our intrinsic dimension estimate with Sinkhorn Divergences. While previous work \cite{chizat2020faster} utilized Richardson extrapolation solely on the regularization parameter $\varepsilon$, we introduce a \textit{diagonal} approach that links $\varepsilon$ to the sample size $n$ and the intrinsic dimension $d_{\text{int}}$. This joint debiasing strategy effectively cancels out both the first-order entropic bias and the statistical discretization error. This scheme permits us to have a convergence of $o(n^{-2/(d_{\text{int}}+4)})$ for $W_2^2$, compared to $o(n^{-2/(d+8)})$ for the $\e$-only Richardson scheme. 
\end{itemize}

We validate our pipeline on synthetic manifolds and real-world datasets (MNIST, CIFAR), demonstrating that our dimension estimator is accurate and that our extrapolation scheme yields Wasserstein estimates with significantly reduced bias compared to standard baselines.

\paragraph{Notations. } 
We note $\|\cdot\|$ the Euclidean norm and $\omega_d$ the volume of the unit $d$-dimensional ball. For $a, b \in \RR$, $a \vee b := \max\{a, b\}$ and $a \wedge b := \min \{a, b\}$. $\P(\RR^d)$ is the set of probabilities in $\RR^d$, and for $\rho \in \P(\RR^d)$, $\mathrm{Supp}(\rho)$ is its support. $\Delta_d$ represents the probability simplex in $\RR^d$: $\Delta_d := \{ w \in \RR_+^d : \sum_{i=1}^d w_i = 1 \}$. $\mathcal{O}(\cdot)$ and $o(\cdot)$ are the usual approximation orders. We use $f \lesssim g$ if there exists a universal constant $C > 0$ such that $f(\cdot) \leq C g(\cdot)$. We write $a \asymp b$ if both $a \lesssim b$ and $b \lesssim a$.

\section{Preliminaries}
\label{sec:preliminaries}

\subsection{Optimal Transport: Primal, Dual, and Semi-Dual}
We consider probability measures supported on a compact metric space $\mathcal{X} \subset \mathbb{R}^d$. Given source and target measures $\mu, \nu \in \mathcal{P}(\mathcal{X})$ and a cost function $c : \mathcal{X} \times \mathcal{X} \to \mathbb{R}_+$, the Kantorovich optimal transport problem is defined as:
\begin{align}\label{def::kantorovich}
   \mathrm{OT}_c(\mu, \nu) := \min_{\pi \in \Pi(\mu, \nu)} \int_{\mathcal{X} \times \mathcal{X}} c(x, y) d\pi(x,y),
\end{align}
where $\Pi(\mu, \nu)$ denotes the set of joint distributions with marginals $\mu$ and $\nu$. When $c(x,y) = \|x - y\|^p$ for $p \ge 1$, the quantity $\mathrm{OT}_c(\mu, \nu)^{1/p}$ defines the $p$-Wasserstein distance, denoted $W_p(\mu, \nu)$ \cite{villani2009optimal}.

\paragraph{Dual and Semi-Dual Formulations.}
The Kantorovich problem admits a dual formulation involving continuous potential functions $f, g \in C(\mathcal{X})$:
\begin{equation}\label{eq:dual_ot}
    \mathrm{OT}_c(\mu, \nu) = \sup_{\substack{f, g \\ f(x) + g(y) \le c(x,y)}} \int f  d\mu + \int g  d\nu.
\end{equation}
By defining the $c$-transform of a potential $g$ as $g^c(x) := \inf_{y} [c(x, y) - g(y)]$, one can eliminate the constraint $f(x) \le c(x,y) - g(y)$ by setting $f = g^c$. This yields the \textbf{semi-dual formulation}, which depends on a single potential $g$:
\begin{equation}\label{eq:semi_dual}
    \mathrm{OT}_c(\mu, \nu) = \sup_{g \in C(\mathcal{X})} \int g^c  d\mu + \int g  d\nu.
\end{equation}
This semi-dual form is crucial for our analysis as, even if $\mu$ is continuous, having $\nu$ discrete converts the constrained continuous problem into a finite dimension problem, a property we will leverage to construct our error estimator. 

\subsection{Entropic Regularization and Sinkhorn Divergences}
To overcome the computational cost of exact OT, one commonly employs Entropic Regularization. We define the regularized objective in its primal form by adding a Kullback-Leibler penalty:
\begin{align}\label{def::entropic_ot}
    \mathrm{OT}_{c, \varepsilon}(\mu, \nu) := \min_{\pi \in \Pi(\mu, \nu)} \int c  d\pi + \varepsilon \mathrm{KL}(\pi | \mu \otimes \nu),
\end{align}
where $\varepsilon > 0$ is the regularization strength. To correct for the bias introduced by entropy (where $\mathrm{OT}_{c, \varepsilon}(\mu, \mu) \neq 0$), we can use the the \textbf{Sinkhorn Divergence}:
\begin{equation}\label{def::sinkhorn_div}
    S_{\varepsilon}(\mu, \nu) := \mathrm{OT}_{c, \varepsilon}(\mu, \nu) - \frac{1}{2}\mathrm{OT}_{c, \varepsilon}(\mu, \mu) - \frac{1}{2}\mathrm{OT}_{c, \varepsilon}(\nu, \nu),
\end{equation}
which, for $c = \|\cdot\|^2 $, is a divergence that interpolates between OT (as $\varepsilon \to 0$) and Maximum Mean Discrepancy (as $\varepsilon \to \infty$) \cite{genevay2018learning, feydy2019interpolating} .

\subsection{Statistical Rates and Intrinsic Dimension}
When approximating $\mu$ and $\nu$ via empirical measures $\hat{\mu}_n$ and $\hat{\nu}_n$, the discretization error was classically shown to scale with the ambient dimension $d$. Informally, for $p$-Wasserstein costs and Sinkhorn Divergences:
\begin{align*}
    \mathbb{E}[|W_p^p(\hat{\mu}_n, \hat{\nu}_n) - W_p^p(\mu, \nu)|] &= \mathcal{O}(n^{-p/d}), \\
    \mathbb{E}[|S_{\varepsilon}(\hat{\mu}_n,\hat{\nu}_n) - S_{\varepsilon}(\mu,\nu) |] &= \mathcal{O}(n^{-1/2}\varepsilon^{-d/2}).
\end{align*}
Crucially, subsequent work established that $d$ is not the ambient dimension, but the \textbf{minimum intrinsic dimension} of the two measures, both for OT \cite{WeedSharpRates, hundrieser2024empirical} and EOT \cite{groppe2024lower, stromme2024minimum}. More precisely, the rate tracks the $\epsilon$-covering number $\mathcal{N}_\epsilon$, defining an effective dimension $d_\epsilon \approx \min_{\rho \in {\mu, \nu}}\frac{\log \mathcal{N}_\epsilon(\rho)}{-\log \epsilon}$. This quantity exhibits a \textit{multi-resolution behavior}: OT and EOT adapt to the geometry at varying scale of observation, capturing $d_\varepsilon$ in the non-asymptotic regime where the discretization error is still significant. This allows the estimator to exploit the effective geometry at coarse scales, in regimes where $d_\epsilon > \lim_{\varepsilon \to 0} d_\epsilon$.

\section{A Simple Estimator of the Wasserstein Discretization Error}
\label{sec:estimator}

As discussed in Section \ref{sec:preliminaries}, the reliability of OT in machine learning hinges on controlling the error incurred when replacing continuous measures with discrete samples. A fundamental insight into this error comes from the triangle inequality of the Wasserstein distance. For any two measures $\mu, \nu$ and their discrete approximations $\hat{\mu}, \hat{\nu}$, the estimation error is bounded by: 
\begin{equation*}
    |W_p(\mu, \nu) - W_p(\hat{\mu}, \hat{\nu})| \leq W_p(\mu, \hat{\mu}) + W_p(\nu, \hat{\nu}).
\end{equation*}
This decomposition reveals that the total error is governed by how well discrete samples capture the geometry of the underlying distributions. While asymptotic theory predicts a rate of $n^{-1/d}$, these bounds offer limited practical guidance for finite datasets. Estimating the term $\mathrm{OT}_c(\rho, \hat{\rho})$ for $\rho= \mu, \nu$ directly from data serves two critical purposes: (i) it provides a concrete, non-asymptotic quantification of the approximation quality; and (ii) the rate at which this error decays with sample size acts as an indicator of the \textit{intrinsic dimension} of the measure.

In this section, we introduce a fast, solver-free estimator for this one-sample discretization error $\mathrm{OT}_c(\rho,\hatrho^*)$, where $\hatrho^*$ represents the optimal discrete approximation of $\rho$ given a fixed support of $n$ points.

\subsection{Wasserstein discretization error and semi-discrete OT}

Consider a measure $\rho$ and a fixed support set $X = (X_1, ..., X_n)$ sampled i.i.d from $\rho$. For any weight vector $\bfa \in \Delta_n$, let $\hatrhoa = \sum_{i=1}^n a_i \delta_{X_i}$ be the corresponding discrete measure. The computation of $\mathrm{OT}_c(\rho,\hatrhoa)$ constitutes a semi-discrete optimal transport problem. Even with continuous $\rho$, this problem admits a finite-dimensional semi-dual formulation \eqref{eq:semi_dual}:
\begin{align}\label{eq::semi_dual_semi_discrete}
    \mathrm{OT}_c(\rho,\hatrhoa)\! = \!\max_{\bfg \in \RR^n} \bigg[ 
    \int_{\RR^d} \bfg^c(x) \d\rho(x) \!+\! \sum_{i=1}^n a_i g_i\bigg].  
\end{align}
Solving \eqref{eq::semi_dual_semi_discrete} typically requires expensive stochastic gradient descent or semi-discrete solvers \cite{genevay2018learning, MerigotNewtonSemiDiscrete}, making it impractical if the goal is merely to evaluate the quality of the discretization when one uses $\OT_c(\hat{\mu}_n, \hat{\nu}_n)$ as a computationally tractable proxy of $\OT_c(\mu, \nu)$ . Remarkably, the following proposition demonstrates that we can bypass optimization entirely. By selecting the optimal weights for the fixed support, the dual potential collapses to the zero vector, yielding a closed-form solution.

\begin{proposition}\label{prop::weight_opti}
    Let $c$ be a continuous cost. Given $\rho \in \P(\RR^d)$ and a support $X=(x_1,\ldots,x_n)\in(\RR^d)^n$, define the weights $\bfw_n=(w_1,\ldots,w_n)^\top$ by
    \[
        w_i \;=\; \rho\bigl(\{x \in \RR^d ;  \min_{j} c(x,x_j)=c(x,x_i)\}\bigr),
    \]
     and set $\hat\rho^*_n=\sum_{i=1}^n w_i \delta_{x_i}$. Then $\hat\rho_n^*$ provides the best approximation of $\rho$ supported on $X$, that is,
    \[\hat
        \rho_n^* \in \argmin_{\bfa\in\Delta_n}\Big\{ \OT_c(\rho,\hat\rho_n^\bfa)\ \big|\ \hat\rho_n^\bfa=\sum_{i=1}^n a_i \delta_{x_i}\Big\},
    \]
    and the minimizer is unique. Moreover, the zero vector $\mathbf{0}_n$ is the maximizer of the semi-dual \eqref{eq::semi_dual_semi_discrete} at $\hat\rho_n^*$. That is, 
    \[
        \OT_c(\rho,\hat\rho^*_n) \;=\; \int_{\RR^d} \mathbf{0}_n^{c}(x) d\rho(x).
    \]
\end{proposition}
\begin{proof}
    By definition of $\bfw_n$, we have that the vector $\mathbf{0}_n$ is the Kantorovich potential of the semi-discrete OT problem between $\rho$ and $\rho_n$. Indeed, the first-order optimality condition, stating that $\bfg$ is optimal if and only if $\nabla H(\bfg) = \mathbf{0}_n$ reads $\rho(\{ \bfg^c(x) = c(x, X_i) - g_i \}) = w_i$ which is exactly how we defined the weights $w_i$ fixing $g_i=0$, $i\in \llbracket 1,n \rrbracket$.  

    Since $\mathbf{0}_n$ is its Kantorovich potential, the optimal transport reads for all $x \in \RR^d : x \mapsto X_i = \argmin_{X_i}c(x, X_i)$.  That is, every point is sent to its closest neighbor from in $X$. For any other probability weight vector, another $\bfg' \in \RR^n$ such that $\bfg' \notin \{ \lambda \mathbf{1}_n , \lambda \in \RR \}$ will be an optimal vector, which will lead to an OT plan where a strictly positive mass of $\rho$ will not be sent to its nearest neighbor, therefore increasing the OT cost. 
\end{proof}

\subsection{A Solver-Free Monte Carlo Estimator}

The key insight of Proposition \ref{prop::weight_opti} is that for the optimally weighted measure $\hat\rho^*_n$, the semi-dual objective is maximized at $\mathbf{0}_n$. This simplifies the OT cost to the integral of the $c$-transform of the zero vector, which is simply the distance to the nearest neighbor in the support $X$. Consequently, we can approximate $\OT_c(\rho,\hat\rho^*_n)$ using a direct Monte Carlo (MC) estimator $\hat{\OT}_{N}^n$ using $N$ additional samples $X_1,\ldots,X_N \stackrel{\text{i.i.d.}}{\sim} \rho$:
\begin{align}\label{eq:estimator_def}
  \hat{\OT}_{N}^n := \frac{1}{N}\sum_{k=1}^N \mathbf{0}_n^{c}(X_k) \;=\;  \dfrac1N\sum_{k=1}^N\min_{j \in \llbracket 1,n \rrbracket}  c(X_k,x_j).
\end{align}
This estimator is computationally efficient, fully parallelizable on GPUs, and requires no OT solver. Furthermore, if explicit weights are required, they also admit a natural MC estimator:
\[
    w_i \;\approx\; \frac{1}{N}\sum_{k=1}^N \mathbf{1}\!\left\{\arg\min_{j} c(X_k,x_j)=i\right\} \;=:\; \hat w_{i,N}.
\]
The following theorem establishes that these estimators concentrate exponentially fast around their true values.

\begin{proposition}\label{prop:discretization_delta_prob}
    Suppose that $\rho$ has a bounded support and that the cost function is continuous.    Then for any $\delta > 0$, with probability at least $1 - \delta$, we have:
     \[
        \left| \OT_c(\rho, \hat \rho_n^*) - \hat{\OT}_N^n  \right|
        \le\underbrace{\sqrt{\tfrac{2\hat{\sigma}_{\bfzero^c}^2\log(2/\delta)}{N}} 
        + \tfrac{7C_\rho\log(2/\delta)}{3(N-1)}}_{E_{\text{MC,OT}}}\ ,
    \]
     
    where $C_\rho= \sup_{x,y \in \text{supp}(\rho) } c(x,y)$ and $\hat{\sigma}_{\bfzero^c}^2$ is the usual biased sample variance estimator $$\dfrac1N\sum_{k=1}^N\Big(\min_{j \in \llbracket 1,n \rrbracket}  c(X_k,x_j)-\hat{\OT}_{N}^n\Big)^2.$$  Also, for any $i \in \llbracket 1, n \rrbracket$, with probability $1-\delta$, we have
     
    \[
        \left| w_i - \hat{w}_{N,i} \right| \le \underbrace{\sqrt{ \tfrac{2 \hat{w}_{N,i} (1 -\hat{w}_{N,i}) \log(2/\delta)}{N} } 
    + \tfrac{7\log(2/\delta)}{3(N-1)}}_{E_{\text{MC},w_i}}\ .
    \]
\end{proposition}
\begin{proof}
    Under our assumptions,  we can bound a.s. $\mathbf{0}_n^c(\cdot)=\min_{j \in \llbracket 1,n \rrbracket}  c(\cdot,x_j)$  by  $C_\rho$. Therefore, the empirical Bernstein bound in Theorem 4 in \cite{MaurerPontil2009} yields the first inequality. The same theorem applies to the second inequality, since $w_i \in [0,1]$  for all i. 
\end{proof}

We conclude with two practical observations. First, although the bound depends on the diameter $C_\rho$, this term appears linearly and does not affect the convergence rate. Second, while standard practice often assumes uniform weights ($\hat{\rho}_n^{\text{unif}} = \frac{1}{m}\sum \delta_{x_i}$), our estimator $\hat{\OT}_{N}^n$ targets the optimally weighted measure $\hat{\rho}_n^*$. Since $\OT_c(\rho, \hat{\rho}_n^*) \leq \OT_c(\rho, \hat{\rho}_n^{\text{unif}})$, our method provides a tight lower bound on the discretization error achievable on the support $X$. More importantly, because $\hat{\rho}_n^*$ optimally adapts the mass to the geometry of the support, the decay rate of $\OT_c(\rho, \hat{\rho}_n^*)$ as $n$ increases is driven purely by the intrinsic geometry of $\rho$. This property makes $\hat{\OT}_{N}^n$ the key building block for our fast intrinsic dimension estimator, detailed in the next section.

\section{Discretization Error and Intrinsic Dimension Estimation}
\label{sec:dimension}

\subsection{Geometric setting: Multi-scale behavior and Covering Numbers}

A recurring theme in machine learning is that high-dimensional data often exhibits low-dimensional structure. While the \emph{manifold hypothesis}, positing that data concentrates on a smooth submanifold $\mathcal{M} \subset \RR^d$, provides a useful baseline, real-world data rarely adheres strictly to such idealized assumptions. Instead, complex distributions often display a multi-resolution behavior, where the effective geometric complexity depends on the scale of observation or the desired approximation accuracy.

To illustrate this, consider $\rho \in \P(\RR^{10})$ constructed as the following mixture: $\rho := \tfrac{1}{2}\Uc([0,1]^{2}\times\{0\}^8) + \tfrac{1}{2}\Uc([1,2]^{8}\times\{0\}^2)$. Strictly speaking, the support has a topological dimension of 8. However, the behavior of the optimal transport discretization error reveals a more nuanced reality. If one aims for a coarse approximation (e.g., a discretization error of 50\%), the problem may behave effectively like a 2-dimensional problem. Conversely, demanding a refined accuracy (e.g., less than 10\% error) forces the discretization to fill the 8-dimensional volume, causing the error rate to suffer from the curse of dimensionality associated with $d=8$.

This example highlights that for probability measures, the relevant notion of complexity is not a static integer but a dynamic quantity related to the \emph{covering numbers} of the support at different radii \cite{WeedSharpRates}. Consequently, we replace rigid geometric definitions (such as the manifold hypothesis) with a functional assumption based on the convergence rate of the Wasserstein distance itself.

We assume that within a specific range of sample sizes $n \in [n_{\min}, n_{\max}]$, the discretization error is governed by an effective intrinsic dimension $\dint$.

\begin{assumption}[Effective Intrinsic Dimension]
\label{assum:rate}
Let $\rho \in \mathcal{P}(\RR^d)$ be a probability measure with finite diameter. We assume that for $n \in [n_{\min}, n_{\max}]$, the convergence rate of the 1-Wasserstein discretization error is governed by a dimension parameter $s = \dint$. Specifically, there exist constants $C_1, C_2 > 0$ depending essentially on $\text{Diam}(\rho)$, such that $C_2\le c C_1$ for a universal constant $c\geq 1$, and:

 \textbf{Lower bound (Quantization limit):} For any discrete measure $\sigma_n$ supported on at most $n$ points, the approximation error is lower-bounded by the dimension:
    \begin{equation}
        W_1(\rho, \sigma_n) \ge C_1 n^{-1/s}.
    \end{equation}
\textbf{Upper bound (Statistical performance):} The empirical measure $\hat{\rho}_n$ constructed from $n$ i.i.d. samples satisfies:
    \begin{equation}
        \mathbb{E}[W_1(\rho, \hat{\rho}_n)] \le C_2 n^{-1/s'}, 
    \end{equation}
for all $s' \geq s$. We refer to the exponent $s$ as the \textbf{intrinsic dimension} $\dint$ of $\rho$ in the regime $[n_{\min}, n_{\max}]$.
\end{assumption}

This assumption posits that $\dint$ captures the "difficulty" of the Wasserstein discretization error of $\rho$ with a finite number of points. In the asymptotic limit ($n_{\min} \to \infty$), if $\rho$ admits a density with respect to a manifold of dimension $d_{\mathcal{M}}$, then $s$ coincides with $d_{\mathcal{M}}$. We refer to the appendix for further discussion, and bounds on the constants in specific settings. 

\subsection{Leveraging our discretization estimator to estimate the intrinsic dimension.}

To estimate the intrinsic dimension, we leverage the scaling behavior of the one-sample discretization error established in Assumption \ref{assum:rate}. Using our Monte Carlo estimator $\hat{\OT}_N^n$ (with $c(x,y)=\|x-y\|$) defined in Section \ref{sec:estimator}, we evaluate the error at two distinct support sizes, $n$ and $\eta n$ (with $\eta > 1$). The ratio of these estimates reveals the exponent governing the convergence rate, leading to the following estimator:
\[
\hat{d}^* \;:=\; \frac{\log \eta}{\log \hat{\OT}_N^n - \log \hat{\OT}_N^{\eta n}},
\]
where $N$ denotes the number of samples used in the MC approximation. Evaluating $\hat{d}^*$ reduces to computing two MC estimates, yielding a linear complexity in $n$. The statistical guarantee of our method is given in the following theorem.

\begin{theorem}\label{th::dim_estim}
    Suppose that Assumption \ref{assum:rate} holds and that $\dint > 2$ and let $\text{Diam}(\rho)$ denote the diameter of $\rho$'s support. Then, for $\gamma, \xi > 0 $ and $n\eta \le n_{\max}$, using
    \begin{align*}
        \eta &\ \geq\  \Big(\frac{(1+\xi)C_2}{C_1}\Big)^{\frac{2\dint}{\gamma}}\ ,\\
        n &\ \geq\  \left(\frac{\mathrm{Diam}(\rho)^2}{2\xi^2C_2}\log(4/\delta)\right)^{\frac{\dint}{\dint - 2}}\ ,\\
        N &\ \geq\ 
        \left[
        \frac{16\,\dint^2}{\gamma^2(\log\eta)^2}
        \; \vee\;1    
        \right]\,\frac{2\,\mathrm{Diam}(\rho)^2}{C_1^2}
        \eta^{2/\dint}n^{2/\dint}
        \log\!\Big(\frac{8}{\delta}\Big)\ ,
    \end{align*}
    we have with probability $1 - \delta$ 
    \begin{align*}
        \frac{\dint}{1 + \gamma} \leq \hat{d}^* \leq (1 + \gamma)\dint \ . 
    \end{align*}
\end{theorem}

\textbf{Comparison to the literature.} 
The concept of leveraging the multi-scale decay of Wasserstein error for dimension estimation was pioneered in \cite{block2022intrinsic}. However, their estimator relies on the two-sample empirical distance:
\[
\widehat d \;=\; \frac{\log \eta}{\log W_1(\hat\rho_n,\hat\rho_n') - \log W_1(\hat\rho_{\eta n},\hat\rho_{\eta n}')}\ ,
\]
where $\hat\rho_n$ and $\hat\rho_n'$ are independent empirical measures of $n$ points. This approach rests on the fact that $\EE[ W_1(\hat\rho_n,\hat\rho_n')] \approx \EE[W_1(\hat\rho_n,\rho)]$, but it faces a severe computational bottleneck: evaluating $W_1(\hat\rho_n,\hat\rho_n')$ requires solving a full discrete optimal transport problem. 
Standard Linear Programming solvers scale as $\mathcal{O}(n^3)$, while approximate Sinkhorn solvers scale as $\mathcal{O}(n^2/\varepsilon)$. This prohibitive cost restricts such estimation to small datasets or requires substantial approximations that may bias the dimension estimate.

In contrast, our estimator $\hat{d}^*$ relies on the one-sample error to the optimal support, which we approximate via Monte Carlo without any optimization. The resulting complexity is linear $\mathcal{O}(n)$, allowing for dimension estimation on large-scale datasets where standard OT solvers are intractable.

\subsection{Experiments}

We compare our estimator, denoted as Semi-Discrete $W_1$, against the Discrete $W_1$ estimator based on \cite{block2022intrinsic} and standard baselines from the \texttt{scikit-dimension} package \cite{bac2021scikit}: MLE, CorrInt, lPCA, TwoNN.

\begin{figure}[ht]
    \centering
    \includegraphics[width=0.5\linewidth]{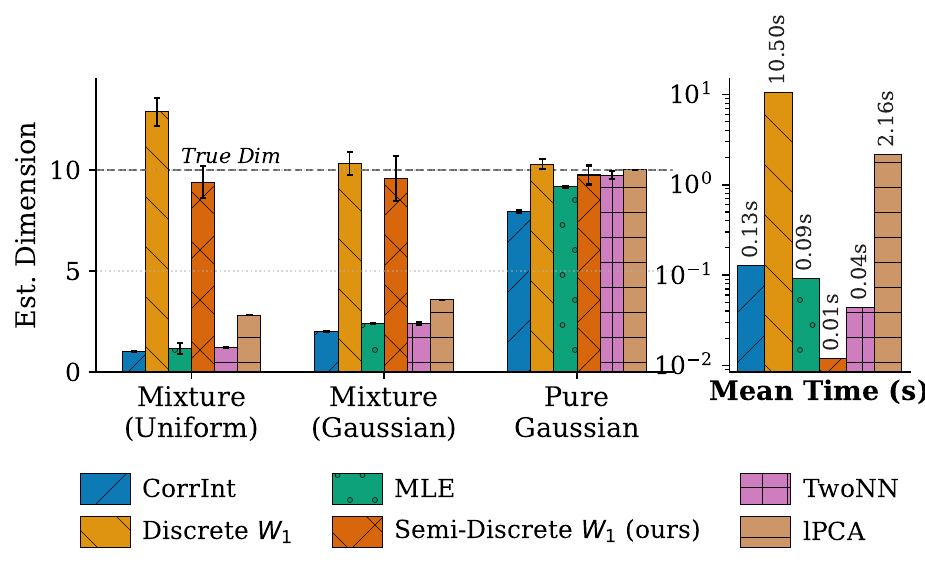}
    \caption{Intrinsic Dimension Estimation Benchmark. Comparison of our Semi-Discrete $W_1$ estimator against the Discrete $W_1$ baseline and standard geometric estimators. The dashed line represents the ground truth effective dimension $\dint = 10$. Our estimator matches the robustness of discrete OT on mixtures but runs orders of magnitude faster ($<0.1$s vs $\sim 10$s).}
    \label{fig:benchmark_dim}
\end{figure}

\textbf{Experimental Setup:} We generate data in $\mathbb{R}^{20}$ under three configurations designed to have an effective Wasserstein dimension of $\dint = 10$: (1) a mixture of hypercubes ($80\%$ 2D, $20\%$ 10D); (2) a mixture of low-rank Gaussians ($80\%$ rank 2, $20\%$ rank 10); and (3) a single rank-10 Gaussian. OT-based estimators compare supports of size $n=2000$ and $\eta n=3000$, while baselines use a fixed size of 3000. Experiments were repeated 20 times.

Figure \ref{fig:benchmark_dim} highlights the performance differences. While all methods succeed on the simple manifold (Config 3), standard geometric estimators fail to capture the global transport complexity of the mixtures (Configs 1 \& 2). In contrast, both OT-based methods robustly recover $\dint \approx 10$. Crucially, our Semi-Discrete estimator is drastically more efficient than the Discrete Wasserstein-based estimator: it computes the estimate in under 0.1 seconds, compared to $\simeq$ 10 seconds on GPU for the Discrete $W_1$ baseline which requires solving exact OT problems. For reference, CPU run times for \texttt{scikit-dimension} baselines are also provided.

\section{Coupled Bias Reduction: The Diagonal Richardson Estimator}

Sinkhorn divergence serves as a fast proxy for OT, supported by efficient solvers like \texttt{geomloss} \cite{feydy2019interpolating} which offer linear memory usage and high GPU performance, provided $\varepsilon$ is not vanishingly small. However, estimating the unregularized Optimal Transport cost from finite samples with the Sinkhorn divergence presents a dual challenge: the estimator is biased by both the regularization parameter $\varepsilon$ (entropic bias) and the finite sample size $n$ (statistical bias).  To further reduce the entropic bias, \cite{chizat2020faster} introduced a Richardson Extrapolation scheme on $\varepsilon$, proposing the estimator $\hat{R}_{\varepsilon, n} = 2\hat{S}_{\varepsilon, n} - \hat{S}_{\sqrt{\varepsilon},n}$. Although this approach theoretically improves the approximation error to $\mathbb{E}[|\hat{R}_{\varepsilon, n} - W_2^2 |] = o(n^{-2/(d+8)})$ (with $\varepsilon \asymp n^{-1/(d+8)}$). The first improvement is naturally to replace $d$ in $\varepsilon$ with our estimator $\widehat d^*$ of $d_{\text{int}}$. Then, even if the sample approximation is not the asymptotic limiting factor, its large constant pre-factor can render entropic debiasing ineffective in practice by targeting a term that is not the dominant source of error.

To better tackle both the entropic and statistical bias, we propose a \textit{coupled} strategy—the \textbf{Diagonal Richardson Estimator}—which links the regularization schedule $\varepsilon$ directly to the sample size $n$. By tackling both biases simultaneously, this method leverages our intrinsic dimension estimator to calibrate the debiasing optimally.

\paragraph{Total bias decomposition.}

It is well known in statistical OT that the entropic regularization should decay with the number of samples; i.e., $\varepsilon_n \asymp n^{-a}$ for some exponent $a > 0$. The total expected error decomposes into:
\begin{align*}
    \mathbb{E} [\hat{S}_{\varepsilon_n, n}] - W_2^2 
    &= \underbrace{\left( \mathbb{E}[\hat{S}_{\varepsilon_n, n}] - S_{\varepsilon_n} \right)}_{\text{Statistical Bias}} 
    + \underbrace{\left( S_{\varepsilon_n} - W_2^2 \right)}_{\text{Entropic Bias}}.
\end{align*}

We characterize the asymptotic expansions of both terms. The entropic bias expansion \eqref{eq::bias_1} follows from Proposition 4 in \cite{chizat2020faster} (assuming finite Fisher information of the measures), while the statistical bias expansion \eqref{eq::bias_2} is derived from Theorem 2 in \cite{stromme2024minimum}, which establishes the lower complexity adaptation of the Sinkhorn divergence:
\begin{align}
    \text{Bias}_{\text{ent}}(n) &= C_3 n^{-2a} + o_{\text{ent}}(n^{-2a})\ , \label{eq::bias_1} \\
    \text{Bias}_{\text{stat}}(n) &\leq C_4 n^{-(1 - a \dint)/2} + o_{\text{stat}}(n^{-(1 - a \dint)/2}) \ . \label{eq::bias_2}
\end{align}

\begin{remark}
    The constant governing the statistical bias is formally of the form $c(\varepsilon)n^{-1/2}$, depending on the second derivative of the Sinkhorn divergence. The upper bound $c(\e) \lesssim \varepsilon^{-\dint/2}$ was obtained independently in \cite{stromme2024minimum, groppe2024lower}, Example 5, through the rate $n^{-(1 - a d_{\epsilon})/2}$ when $\varepsilon \asymp n^{-a}$ on the standard error. This rate is solely governed by the statistical bias since $\hat{S}_{\varepsilon, n}- \mathbb{E} [\hat{S}_{\varepsilon, n}] =\mathcal{O}_\P(1/\sqrt n)$ uniformly in $\varepsilon \ge 0$ in Proposition 4 of \cite{chizat2020faster}. The sharpness of the rate of the statistical bias was not perfectly established, though discussions are provided for our case of interest $\varepsilon = n^{-a}$ and $d_{\epsilon}$ replaced by $\dint$. Therefore, to be completely rigorous, we use the equality in \eqref{eq::bias_2} as an assumption, as it is compulsory for our Richardson method to effectively cancel the bias term.
\end{remark}

\begin{assumption}[Sharpness of Statistical Bias]
\label{assum:sharpness}
    We assume that the statistical estimation error of the Sinkhorn divergence admits a sharp asymptotic expansion. That is, equality holds in \eqref{eq::bias_2}:
    \[
        \mathbb{E}[\hat{S}_{\varepsilon_n, n}] - S_{\varepsilon_n} \;=\; C_4 n^{-(1 - a \dint)/2} + o_{\text{stat}}(n^{-(1 - a \dint)/2})\ ,
    \]
    where $C_4 \neq 0$ and $\dint$ is the intrinsic dimension of the data.
\end{assumption}

\paragraph{Optimal schedule and the Diagonal Estimator.}

To construct an efficient estimator, we seek a rate $\e \asymp n^{-a}$ that balances the convergence rates of the entropic and statistical biases. Equating the exponents $4a = 1 - a \dint$ yields the optimal decay rate $a = \frac{1}{\dint+4}$. Under this schedule, a \textit{diagonal} Richardson extrapolation can eliminate the first-order terms of both biases simultaneously.

We define our Diagonal Richardson Estimator using sample sizes $n$ and $2n$ as:
\[
    \hat{R}^{\text{Diag}}_{2n} := w_{2n} \hat{S}_{\varepsilon_{2n}, 2n} + w_n \hat{S}_{\varepsilon_{n}, n} \ ,
\]
where the weights are determined by the common decay rate $\gamma = \frac{2}{\dint+4}$ of the bias terms:
\[
    w_{2n} = \frac{2^\gamma}{2^\gamma - 1}, \quad w_n = \frac{-1}{2^\gamma - 1}.
\]
The following proposition establishes the improved convergence rate of this coupled estimator.

\begin{proposition}\label{prop::richardson}
    Under Assumption \ref{assum:sharpness}, and using the schedule $\varepsilon_n \asymp n^{-1/(\dint+4)}$, the Diagonal Richardson estimator satisfies:
    \[
        \mathbb{E}\left[|\hat{R}^{\text{Diag}}_{2n} - W_2^2|\right] =  o(n^{-\frac{2}{\dint+4}})\ . 
    \]
    Furthermore, if the higher-order terms satisfy $o_{\text{ent}}(n^{-2a}) = \mathcal{O}(n^{-4a})$ and $o_{\text{stat}}(n^{-(1 - a \dint)}) = \mathcal{O}(n^{-2 + 2a\dint})$, the rate improves to $\mathcal{O}(n^{-\frac{4}{\dint+4}})$.
\end{proposition}
Therefore, this rate is slightly better than the $\varepsilon$-only Richardson, which is $(n^{-\frac{2}{\dint+8}})$ using $\dint$.

\textit{Proof sketch.}
Let $R^{\text{Diag}} = \lim_{n \to \infty}\mathbb{E}[\hat{R}^{\text{Diag}}_{2n}]$. The error decomposes as:
\begin{align*}
    \mathbb{E} &\left[|\hat{R}^{\text{Diag}}_{2n} - W_2^2  |\right] \leq \mathbb{E}\left[ |\hat{R}^{\text{Diag}}_{2n} - \mathbb{E}[\hat{R}^{\text{Diag}}_{2n}]| \right] \\
      &\quad + |\mathbb{E}[\hat{R}^{\text{Diag}}_{2n}] - R^{\text{Diag}}| + |R^{\text{Diag}} - W_2^2|\\
    &= \mathcal{O}(1/\sqrt{n}) +  o_{\text{stat}}(n^{-\frac{2}{\dint+4}} ) + o_{\text{ent}}(n^{-\frac{2}{\dint+4}}).
\end{align*}
The stochastic error scales as $\mathcal{O}(n^{-1/2})$, which is negligible compared to the bias terms for $\dint > 2$. The choice of weights cancels the leading order terms $n^{-\frac{2}{\dint+4}}$ in the bias, leaving only the higher-order remainder. Remark that if Assumption \ref{assum:sharpness} is not fulfilled, we still have the asymptotically (in $\dint$) minimax rate $\mathcal{O}(n^{-2/(\dint +4)})$. 

\begin{figure}[ht]
    \centering
    \includegraphics[width=0.37\linewidth]{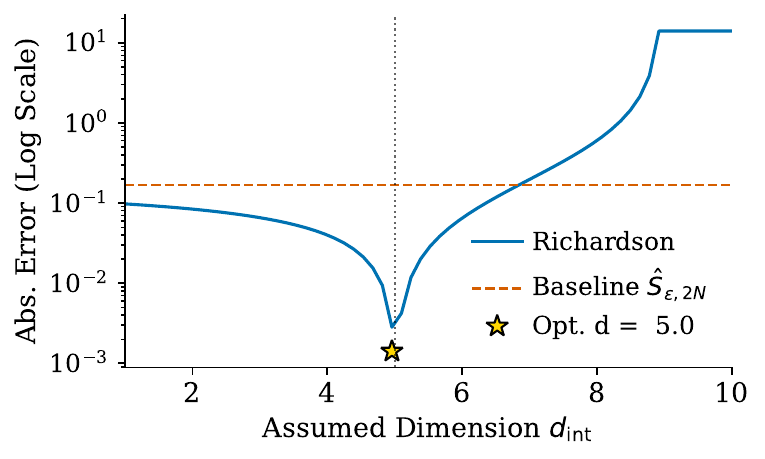}
    \caption{\textbf{Sensitivity to Intrinsic Dimension.} We perform Diagonal Richardson extrapolation, $2n = 2000$, on a source synthetic mixture in $\mathbb{R}^{10}$ (90\% mass on a 5D Gaussian, 10\% on a 1D Gaussian) to a 10D Gaussian. By varying the dimension parameter $d$ used in the extrapolation weights, we observe that the estimation error is minimized exactly at the dominant intrinsic dimension $d=5$. This confirms that correctly calibrating the schedule to the effective geometry is essential for optimal debiasing.}
    \label{fig:richardson_best_d}
\end{figure}

\paragraph{Bagged Diagonal Richardson.}

While the Diagonal Richardson estimator $\hat{R}^{\text{Diag}}_{2n}$ successfully reduces the bias, this improvement comes at the cost of increased variance. This variance increase comes from extrapolation involving a weighted difference from the coefficients $w_{2n}$ and $w_n$. Consequently, the noise from the subsampled term $\hat{S}_{\varepsilon_n, n}$ is amplified in the final estimate. Furthermore, relying on a single random subset of size $n$ is statistically inefficient, as it discards information from the unused samples for the low-resolution component. To mitigate this variance increase, we employ a bagging strategy that averages the low-resolution estimator over $K$ independent subsamples, while retaining the full dataset for the high-resolution term. The \emph{Bagged Richardson Estimator} is defined as:
\begin{equation}
    \hat{R}^{\text{Bagged}}_{K, 2n} := w_{2n} \hat{S}_{\varepsilon_{2n}, 2n} + \frac{w_n}{K} \sum_{k=1}^K \hat{S}_{\varepsilon_{n}^{(k)}, n} \ .
\end{equation}

This approach stabilizes the estimator by driving down the variance of the subtractive term. We quantify this stability in terms of the variance inflation relative to the standard Sinkhorn estimator. The proof, detailed in the Appendix, builds on the second-order Hadamard differentiability of the Sinkhorn divergence established in \cite{goldfeld2024limit}.

\begin{proposition}[Variance Stability]\label{prop::var_reduc}
The variance of the Bagged Richardson estimator relates to the variance of the standard Sinkhorn estimator $\hat{S}_{\varepsilon_{2n}, 2n}$ as:
\begin{equation*}
    \mathrm{Var}(\hat{R}^{\text{Bagged}}_{K,2n}) =  \mathrm{Var}(\hat{S}_{\varepsilon_{2n}, 2n}) \times \left(1 + \tfrac{1}{(2^{\gamma}-1)^2 K}\right) + o(1).
\end{equation*}
\end{proposition}
Crucially, as the number of bags $K \to \infty$, the variance overhead vanishes. The bagged estimator thus achieves the improved bias rate of the Richardson scheme while converging to the optimal variance floor dictated by the full sample size $2n$.

\begin{figure}[ht]
    \centering
    \includegraphics[width=0.44\linewidth]{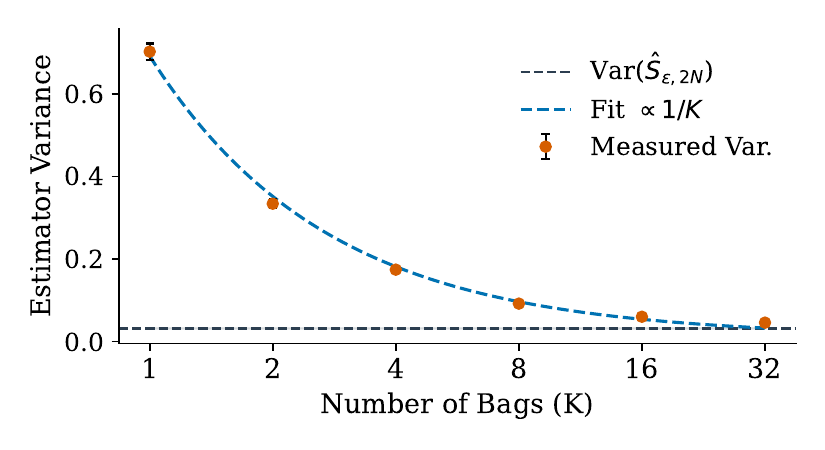}
    \caption{\textbf{Variance Reduction via Bagging.} Using the same mixture setting as in Figure \ref{fig:richardson_best_d}, we measure the variance of the Bagged Richardson estimator as a function of the number of bags $K$. The variance is computed over 20 runs and repeated 10 times. We observe a clear $1/K$ decay in the variance overhead, which converges asymptotically to the variance floor of the standard Sinkhorn estimator $\hat{S}_{\varepsilon_{2n}, 2n}$ (dashed line). This confirms that bagging eliminates the stability cost of extrapolation.}
    \label{fig:variance_decay_bagging}
\end{figure}

\subsection{From OT Geometry to Wasserstein Distance Estimation}

We validate our full "geometry-aware" pipeline by estimating $W_2^2$ on image manifolds (MNIST, FashionMNIST, CIFAR10). To overcome the lack of ground truth for empirical measures, we construct synthetic target measures $\nu = (T^*)_\# \mu$ using a strictly convex potential $\psi_\theta(x) = \frac{1}{2}\|x\|^2 + \text{ICNN}(x)$. By Brenier's theorem, $T^* = \nabla \psi_\theta$ is the unique optimal map, providing an exact reference distance $W_2^2(\mu, \nu) = \mathbb{E}[\|x - T^*(x)\|^2]$.

\textbf{Experimental Setup.} For each dataset, with $n=1000$, we first estimate the intrinsic dimension $\hat{d}^*$ (Sec. \ref{sec:dimension}) and plug it into our Bagged Diagonal Richardson estimator ($K=12$) to calibrate the bias cancellation weights. We compare our approach against the standard Sinkhorn divergence ("Base"),  and the standard $\varepsilon$-Richardson extrapolation ("Eps-Rich") \cite{chizat2020faster}.
\begin{table}[H]
\centering
\caption{\textbf{Benchmark of Wasserstein estimation} $W_2^2$\textbf{.}}
\label{tab:wasserstein_results}
\footnotesize
\renewcommand{\arraystretch}{0.9}
\setlength{\tabcolsep}{4pt}
\begin{tabular}{llcc}
\toprule
Dataset & Method & Error $\pm$ std & Avg Time(s) \\
\midrule
\multirow{4}{*}{CIFAR10} 
    & Base     & 1.210 $\pm$ 0.02 & 0.94  \\
    & n-Diag-Rich   & \textbf{0.089} $\pm$ 0.03 & 2.85  \\
    & $\tfrac{n}{4}$-Diag-Rich & 0.201 $\pm$ 0.03 & 0.89  \\
    & Eps-Rich & 1.007 $\pm$ 0.02 & 1.56  \\
\midrule
\multirow{4}{*}{\shortstack[l]{Fashion\\MNIST}} 
    & Base     & 0.852 $\pm$ 0.02 & 0.84 \\
    & n-Diag-Rich   & \textbf{0.259} $\pm$ 0.02 & 2.89 \\
    & $\tfrac{n}{4}$-Diag-Rich  & 0.355 $\pm$ 0.04 & 0.91  \\
    & Eps-Rich & 0.849 $\pm$ 0.02 & 1.26 \\
\midrule
\multirow{4}{*}{MNIST}   
    & Base     & 0.788 $\pm$ 0.01 & 0.78 \\
    & n-Diag-Rich   & \textbf{0.219} $\pm$ 0.02 & 3.30  \\
    & $\tfrac{n}{4}$-Diag-Rich  & 0.420 $\pm$ 0.04 & 0.95  \\
    & Eps-Rich & 0.706 $\pm$ 0.01 & 1.40  \\
\bottomrule
\end{tabular}
\end{table}

\textbf{Results.} Table \ref{tab:wasserstein_results} confirms that incorporating intrinsic geometry is essential for accurate estimation. Standard $\varepsilon$-extrapolation fails to address the statistical bias dominating in high dimensions, yielding errors comparable to the baseline, even when using $\hat{d}^*$. In contrast, our diagonal estimator, calibrated via $\hat{d}^*$, reduces the error by an order of magnitude (e.g., $1.21 \to 0.09$ on CIFAR10). This validates that coupling the regularization schedule to the intrinsic dimension ($\varepsilon \asymp n^{-1/(\hat{d}^*+4)}$) effectively mitigates the discretization bottleneck. Furthermore, even with a reduced budget (reduced sample size $\tfrac{n}{4}$ - bagging $K=12$), our method significantly outperforms baselines while matching their computational cost.

\section{Conclusion and Future Work}

In this work, we have established that explicitly estimating the intrinsic dimension of data is key to overcoming the statistical limits of discrete Optimal Transport. We showed that our Diagonal Richardson estimator, calibrated via a solver-free intrinsic dimension estimate, significantly improves the accuracy of Wasserstein distance estimation on high-dimensional manifolds. This framework opens several directions for future research. First, the validity of our bias cancellation relies on the sharpness of the statistical error expansion (Assumption \ref{assum:sharpness}), a property that warrants further theoretical analysis. Second, while we focused on the scalar transport cost, extending this "geometry-aware" debiasing strategy  and analysis to Kantorovich potentials and gradients remains a promising avenue for improving training stability in generative modeling.

\bibliography{references}
\bibliographystyle{icml2026}

\newpage

\appendix

\onecolumn

\section{Intrinsic Dimension and Wasserstein Distances}

We recall some notions and properties related to the intrinsic dimension and multi-resolution behavior of the Wasserstein dimension and Sinkhorn divergences, which are present in \cite{WeedSharpRates, stromme2024minimum}.

\subsection{Notions. }

\medskip
\noindent \textbf{Definition ($\delta$-covering number of a set).} Given a set $S \subseteq X$, the $\delta$-covering number of $S$, denoted $\mathcal{N}_\delta(S)$, is the minimum $m$ such that there exists $m$ closed balls $B_1, \dots, B_m$ of diameter $\delta$ such that $S \subseteq \bigcup_{1 \le i \le m} B_i$. The $\delta$-dimension of $S$ is the quantity
\[
d_\delta(S) := \frac{\log \mathcal{N}_\delta(S)}{- \log \delta} \,.
\]

\noindent The following definition is particularly useful when working with measures instead of sets, allowing one to ignore a small fraction of the mass. The following definition appears first in \cite{dudley1969speed}. 

\medskip
\noindent \textbf{Definition ($(\delta, \tau)$-covering number of a measure).} Given a measure $\mu$ on $X$, the $(\delta, \tau)$-covering number is
\[
\mathcal{N}_\delta(\mu, \tau) := \inf \{ \mathcal{N}_\delta(S) : \mu(S) \ge 1 - \tau \}
\]
and the $(\delta, \tau)$-dimension is
\[
d_\delta(\mu, \tau) := \frac{\log \mathcal{N}_\delta(\mu, \tau)}{- \log \delta} \,.
\]

\noindent The limits are defined by
\[
\mathcal{N}_\varepsilon(\mu) := \mathcal{N}_\varepsilon(\mu, 0) \,, \quad
d_\varepsilon(\mu) := d_\varepsilon(\mu, 0) \,.
\]

\medskip
\noindent \textbf{Definition (Wasserstein dimensions).} The upper and lower Wasserstein dimensions are respectively
\begin{align*}
d^*_p(\mu) &= \inf \left\{ s \in (2p, \infty) : \limsup_{\delta \to 0} d_\delta \left(\mu, \delta^{\frac{sp}{s-2p}} \right) \le s \right\}, \\
d_*(\mu) &= \lim_{\tau \to 0} \liminf_{\delta \to 0} d_\delta(\mu, \tau) \,.
\end{align*}

\subsection{Properties}

\begin{proposition}[Prop.5 and Corollary 1 \citep{WeedSharpRates}]\label{prop:weedsharp}
Let $p \in [1, \infty)$. Suppose there exist $\delta' \leq 1$ and $s > 2p$ such that $d_\delta(\mu, \delta^{\frac{sp}{s-2p}}) \leq s$ for all $\varepsilon \leq \delta'$. Then there exist constants $C_1, C_2$ such that:
\begin{equation}
    \mathbb{E}[W_p^p(\mu, \hat{\mu}_n)] \leq C_1 n^{-p/s} + C_2 n^{-1/2}.
\end{equation}
In particular, $\mathbb{E}[W_p(\mu, \hat{\mu}_n)] \lesssim n^{-1/s}$.
As a corollary, if $s > d_p^*(\mu)$, we have $\mathbb{E}[W_p(\mu, \hat{\mu}_n)] \lesssim n^{-1/s}$.
\end{proposition}

We now derive a slightly different proposition, based on Proposition 6 of \cite{WeedSharpRates}, with a similar proof. 
\begin{proposition}[Local Lower Bound for Wasserstein Rates]
\label{prop:local_lower_bound}
Let $\mu$ be a probability measure on $\mathbb{R}^D$ and let $p \in [1, \infty)$. Suppose there exist constants $\tau \in (0, 1)$, $A > 0$, $s > 0$, and a scale interval $[a, b] \subset (0, 1]$ such that the $(\delta, \tau)$-packing number satisfies
\begin{equation}
    \mathcal{N}'_\delta(\mu, \tau) \geq A \delta^{-s}, \quad \forall \delta \in [a, b].
\end{equation}
If the sample size $n$ satisfies the validity condition $A b^{-s} < n < A a^{-s}$, then for any discrete measure $\nu$ supported on at most $n$ points (including the empirical measure $\hat{\mu}_n$), we have the following lower bound:
\begin{equation}
    W_p^p(\mu, \nu) \geq \left( \tau 4^{-p} A^{p/s} \right) n^{-p/s}.
\end{equation}
\end{proposition}

\begin{proof}
Let $n$ satisfy the validity condition. We define the critical radius $\varepsilon_n := (n/A)^{-1/s}$. By the condition $A b^{-s} < n < A a^{-s}$, it follows explicitly that $a < \varepsilon_n < b$.

Consider an arbitrary $\varepsilon < \varepsilon_n$ such that $\varepsilon \in [a, b]$. By the packing assumption, the number of disjoint balls of radius $\varepsilon$ that can be packed into the support of $\mu$ (capturing mass $1-\tau$) satisfies $N := \mathcal{N}'_\varepsilon(\mu, \tau) \geq A \varepsilon^{-s}$. Since $\varepsilon < \varepsilon_n$, we strictly have
\begin{equation}
    N \geq A \varepsilon^{-s} > A \varepsilon_n^{-s} = n.
\end{equation}
Let $\{B_1, \dots, B_N\}$ denote this packing of disjoint balls. Since $\nu$ is supported on at most $n$ points, by the Pigeonhole Principle, at least $N-n$ of these balls do not contain any support points of $\nu$. Let $\mathcal{I}_{\text{empty}}$ denote the indices of these empty balls.

The Wasserstein distance $W_p^p(\mu, \nu)$ involves transporting the mass from these empty balls to the support of $\nu$. For any ball $B_i$ with $i \in \mathcal{I}_{\text{empty}}$, the mass $\mu(B_i)$ must be moved a distance of at least $\varepsilon$ to reach any point in $\text{supp}(\nu)$ (which lies outside $B_i$). Following the sharp derivation in \citet{WeedSharpRates} (Proposition 6), the aggregated transport cost is lower bounded by:
\begin{equation}
    W_p^p(\mu, \nu) \geq \tau \left(\frac{\varepsilon}{4}\right)^p = \tau 4^{-p} \varepsilon^p.
\end{equation}
Since this inequality holds for any $\varepsilon < \varepsilon_n$, we take the limit $\varepsilon \to \varepsilon_n$. Substituting $\varepsilon_n = (n/A)^{-1/s}$ yields:
\begin{equation}
    W_p^p(\mu, \nu) \geq \tau 4^{-p} \left[ (n/A)^{-1/s} \right]^p = \tau 4^{-p} A^{p/s} n^{-p/s}.
\end{equation}
\end{proof}

\subsection{Minimum Intrinsic Dimension scaling}

We recall here the \emph{Minimum Intrinsic Dimension} scaling of the Sinkhorn divergence, which we leverage for the Wasserstein distance estimation.

\textbf{Theorem} (MID scaling for the Sinkhorn divergences). For numerical constants independent of all problem parameters,
\[
\mathbb{E}[|S_{\varepsilon}(\hat{\mu}, \hat{\nu}) - S_{\varepsilon}(\mu, \nu)|] \lesssim (1 + \varepsilon) \sqrt{\frac{\mathcal{N}(\mu, \frac{\varepsilon}{L}) \wedge \mathcal{N}(\nu, \frac{\varepsilon}{L})}{n}}.
\]

The dimensional quantity in this estimate is contained in the minimum covering numbers at scale $\varepsilon$, demonstrating the MID scaling phenomenon. Here, the constant $L$ refers to the Lipschitz bound on the cost, which is assured to be finite for $\|\cdot \|^p$ costs, in the bounded setting.

\paragraph{Connection to Intrinsic Dimension.}
In this work, we characterize the intrinsic dimension at a fixed scale through the convergence rate of $\mathbb{E}[W_1(\mu, \hat{\mu}_n)]$ and its corresponding lower bounds. We now establish the link between this characterization and the MID scaling bound for Sinkhorn divergences presented above.

Consider the cost function $c = \|\cdot\|$, which implies a Lipschitz constant $L = 1$. Suppose that over a scale interval $\varepsilon \in [\varepsilon_a, \varepsilon_b]$, the measure is sufficiently regular such that the dimension is stable across mass thresholds; specifically, $d_{\varepsilon}(\mu, \tau) = d_{\varepsilon}(\mu, 0)$ for all $\tau \in (0, 1/2]$. Under these conditions, Propositions \ref{prop:weedsharp} and \ref{prop:local_lower_bound} provide matching upper and lower bounds on the Wasserstein distance convergence rate with $\dint = d_{\e}$ for a range $n \in[ N_{\min}, N_{\max}]$. This allows us to identify the intrinsic dimension $d_{\text{int}}$ directly with the scale-dependent covering dimension $d_{\varepsilon}$, which is precisely the quantity governing the MID scaling in Theorem 2.

\section{Proof of Theorem \ref{th::dim_estim}: Dimension Estimation Control}
\begin{proof}
Fix $n\ge n_{\min}$ and $\eta>1$. For each $k\in\{n,\eta n\}$, we introduce the notations
\[
W^*_k := \OT_c(\rho,\widehat\rho_k^*),
\qquad
\widehat W_k := \widehat{\OT}_N^k
= \frac{1}{N}\sum_{i=1}^N \min_{1\le j\le k} c(X_i,x_j),
\]
where $x_1,\dots,x_k\stackrel{\mathrm{i.i.d.}}{\sim}\rho$ define the support
of $\widehat\rho_k^*$ and $X_1,\dots,X_N\stackrel{\mathrm{i.i.d.}}{\sim}\rho$
are independent Monte Carlo samples (independent of the $x_j$'s).
Set $c(\cdot)=\|\cdot\|$ and notice that 
$
\mathrm{Diam}(\rho) := \sup_{x,y\in\mathrm{Supp}(\rho)} c(x,y)$.

\paragraph{Monte Carlo concentration for $\widehat W_k$.}
Conditionally on $(x_1,\dots,x_k)$, the random variables
$\min_{j\le k} c(X_i,x_j)$ are i.i.d. and lie in $[0,\mathrm{Diam}(\rho)]$.  Hence,
Hoeffding's inequality gives, for any $t>0$,
\[
\PP\left(\,|\widehat W_k - W^*_k|\ge t\,\right)
\le 2\exp\left(-\frac{2Nt^2}{\mathrm{Diam}(\rho)^2}\right).
\]
Define the Monte Carlo accuracy level
\[
\kappa_N := \mathrm{Diam}(\rho) \sqrt{\frac{\log(8/\delta)}{2N}}.
\]
Then $\PP(|\widehat W_k-W^*_k|>\kappa_N)\le \delta/4$ for each $k$.
By a union bound over $k\in\{n,\eta n\}$, the event
\[
\mathcal E_{\mathrm{MC}}
:= \big\{ |\widehat W_n-W^*_n|\le \kappa_N\big\}
\cap \big\{ |\widehat W_{\eta n}-W^*_{\eta n}|\le \kappa_N\big\}
\]
satisfies
\[
\PP(\mathcal E_{\mathrm{MC}})\ge 1-\delta/2.
\]

\paragraph{Two-sided control of $W^*_k$ at scales $n$ and $\eta n$.}
From the bounds provided in Assumption \ref{assum:rate}, we have for all $k \in [n_{\min}, n_{\max}/\eta]$,
\begin{equation}\label{eq:Wk_lower_as}
W^*_k \ge C_1 k^{-1/\dint}\qquad\text{a.s.}
\end{equation}
and for any $t>0$, using the boundness of the semi-dual optimize $\mathbf{0}_c \leq \mathrm{Diam}(\rho)$, 
\[
\PP\left( W^*_k\ge C_2 k^{-1/\dint}+t\right)
\le \exp\left(-\frac{2kt^2}{\mathrm{Diam}(\rho)^2}\right),
\]
where $\mathrm{Diam}(\rho)$ is the diameter of $\mathrm{Supp}(\rho)$.
Taking $t=\xi C_2 k^{-1/\dint}$ for $\xi > 0$ yields
\[
\PP\left( W^*_k\ge (1 + \xi)C_2 k^{-1/\dint}\right)
\le \exp\left(-\frac{2k\,\xi^2 C_2^2\,k^{-2/\dint}}{\mathrm{Diam}(\rho)^2}\right)
= \exp\left(-\frac{2\xi^2 C_2^2}{\mathrm{Diam}(\rho)^2}\,k^{1-2/\dint}\right).
\]

For $k^{1 - 2/\dint} \geq \max \left\{ {n^*}^{1 - 2/\dint}, \frac{\mathrm{Diam}(\rho)^2}{2\xi^2 C_2^2}\log(4/\delta) \right\}$ , the event
\[
\mathcal E_{\mathrm{Q}}
:= \{W^*_k\le (1+\xi)C_2 k^{-1/\dint}\}\cap\{W^*_{\eta k}\le (1+\xi)C_2 (\eta k)^{-1/\dint}\}
\]
satisfies $\PP(\mathcal E_{\mathrm{Q}})\ge 1-\delta/2$.
On $\mathcal E_{\mathrm{Q}}$ and using \eqref{eq:Wk_lower_as},
\begin{equation}\label{eq:Wk_two_sided}
C_1 k^{-1/\dint}\le W^*_k \le (1+\xi)C_2 k^{-1/\dint},
\qquad k\in\{n,\eta n\}.
\end{equation}

Let $\mathcal E:=\mathcal E_{\mathrm{MC}}\cap \mathcal E_{\mathrm{Q}}$.
Then $\PP(\mathcal E)\ge 1-\delta$.

\paragraph{Control of the population log-ratio $D$.}
Define
\[
D := \log W^*_n - \log W^*_{\eta n}.
\]
From \eqref{eq:Wk_two_sided},
\[
\frac{W^*_n}{W^*_{\eta n}}
\in\left[\frac{C_1 n^{-1/\dint}}{(1+\xi)C_2(\eta n)^{-1/\dint}},\;
\frac{(1+\xi)C_2 n^{-1/\dint}}{C_1(\eta n)^{-1/\dint}}\right]
=
\left[\frac{C_1}{(1+\xi)C_2}\eta^{1/\dint},\;
\frac{(1+\xi)C_2}{C_1}\eta^{1/\dint}\right].
\]
Taking logs and denoting $
D_0 := \log\eta/\dint $ gives
\[
D \in \big[D_0-b,\; D_0+b\big],
\qquad
b:=\log\Big(\frac{(1+\xi)C_2}{C_1}\Big)\ge 0,
\]
hence $|D-D_0|\le b$.
Moreover, the stated condition on $\eta$ implies $b\le \frac{\gamma}{2}D_0$.

\paragraph{Stability of $\log \widehat W_k$ under Monte Carlo error.}
On $\mathcal E$, we have $|\widehat W_k-W^*_k|\le \kappa_N$.
We first ensure positivity of $\widehat W_{\eta n}$ (and similarly $\widehat W_n$)
by requiring
\[
\kappa_N \le \frac12 W^*_{\eta n}.
\]
Using $W^*_{\eta n}\ge C_1(\eta n)^{-1/\dint}$ from \eqref{eq:Wk_lower_as},
it is enough that
\[
\kappa_N \le \frac12 C_1(\eta n)^{-1/\dint},
\]
which is implied by the second lower bound on $N$ in the theorem.

Assume this condition holds. Then for $k\in\{n,\eta n\}$,
\[
\left|\log \widehat W_k - \log W^*_k\right|
= \left|\log\left(1+\frac{\widehat W_k-W^*_k}{W^*_k}\right)\right|.
\]
Since $|\widehat W_k-W^*_k|\le \kappa_N\le \tfrac12 W^*_k$, we have
$\left|\frac{\widehat W_k-W^*_k}{W^*_k}\right|\le \frac12$ and thus
$|\log(1+u)|\le 2|u|$ for $|u|\le 1/2$ yields
\[
\left|\log \widehat W_k - \log W^*_k\right|
\le 2\frac{|\widehat W_k-W^*_k|}{W^*_k}
\le 2\frac{\kappa_N}{W^*_k}.
\]
Therefore, with $\widehat D:=\log\widehat W_n-\log\widehat W_{\eta n}$,
\begin{equation}\label{eq:Dhat_minus_D_fixed}
|\widehat D - D|
\le 2\kappa_N\left(\frac{1}{W^*_n}+\frac{1}{W^*_{\eta n}}\right).
\end{equation}
Using $W^*_n\ge C_1 n^{-1/\dint}$ and $W^*_{\eta n}\ge C_1(\eta n)^{-1/\dint}$,
\[
2\kappa_N\left(\frac{1}{W^*_n}+\frac{1}{W^*_{\eta n}}\right)
\le \frac{2\kappa_N}{C_1}\,n^{1/\dint}\big(1+\eta^{1/\dint}\big)
\le \frac{4\kappa_N}{C_1}\,\eta^{1/\dint} n^{1/\dint}.
\]
Plugging $\kappa_N=\mathrm{Diam}(\rho)\sqrt{\log(8/\delta)/(2N)}$, we get
\[
|\widehat D-D|
\le \frac{4\mathrm{Diam}(\rho)}{C_1}\,\eta^{1/\dint} n^{1/\dint}
\sqrt{\frac{\log(8/\delta)}{2N}}.
\]
The first lower bound on $N$ in the theorem ensures that the RHS is at most
$\frac{\gamma}{2}D_0$.

\paragraph{Denominator interval and propagation to $\hat{d}^*$.}
On $\mathcal E$ we have
\[
|\widehat D-D_0|
\le |\widehat D-D| + |D-D_0|
\le \frac{\gamma}{2}D_0 + \frac{\gamma}{2}D_0
= \gamma D_0,
\]
so $(1-\gamma)D_0\le \widehat D\le (1+\gamma)D_0$.
Finally, we use the relation  $d_N^*= \log\eta/ \widehat D$ to obtain
\[
\frac{\log\eta}{(1+\gamma)D_0} \le \hat{d}^* \le \frac{\log\eta}{(1-\gamma)D_0}
\quad\Longleftrightarrow\quad
\frac{\dint}{1+\gamma}\le \hat{d}^*\le \frac{\dint}{1-\gamma}.
\]
Since $\PP(\mathcal E)\ge 1-\delta$, this completes the proof.

\textbf{Remark:} Observe that, for the proof to hold, we required the condition $\eta n \leq n_{\max}$ to hold. 
\end{proof}

\subsection{Intrinsic Dimension estimation under the manifold hypothesis}
We briefly discuss here, how our proof stands under the manifold hypothesis assumption, as in \cite{block2022intrinsic}. We recall here their assumption. 

\begin{assumption}[Geometric regularity]
\label{assum:geom}
Let ${\cal M}\subset \RR^d$ be a compact manifold of dimension $d_{\cal M}$ with diameter $\Delta=\mathrm{diam}({\cal M})$ and reach $\tau>0$. Furthermore, suppose that the data-generating distribution $\rho$ admits a density with respect to the uniform measure on ${\cal M}$, bounded as $0<w_{\min}\le \tfrac{\d\rho}{\d\mathrm{vol}_{\cal M}}\le w_{\max}<\infty$. 
\end{assumption}

The positive reach $\tau$ guarantees that ${\cal M}$ is free of self-intersections and has controlled curvature; informally, below the scale $\tau$ the nearest-point (normal) projection onto ${\cal M}$ is single-valued, so ${\cal M}$ behaves locally like $\RR^{d_{\cal M}}$. Formally, the reach of ${\cal M}$ is
\begin{align*}
    \mathrm{reach}({\cal M})
:= \sup\Big\{ r>0 \ :\ \forall z\in\RR^D,\ \mathrm{dist}(z,{\cal M})<r \  \Rightarrow\ \operatorname*{argmin}_{y\in {\cal M}}\|z-y\|\ \text{is unique} \Big\}.
\end{align*}

This notion is standard in geometric inference, and the assumption $\tau>0$ holds, for instance, if ${\cal M}\subset\RR^d$ is a compact, embedded $C^{1,1}$ submanifold. We refer to \cite{block2022intrinsic} for more details on the notion of reach and its role in OT-based intrinsic-dimension estimation. Here, we note $\omega_{d_{\mathcal{M}}}$ the volume of the unit $d$-dimensional ball.

\begin{proposition}[(\cite{block2022intrinsic})Upper and lower bound on the Wasserstein error under the manifold hypothesis]\label{prop::from_block}
For 
\[
n > \frac{d_{\cal M} \, \mathrm{vol} \, {\cal M}}{4 \omega_{d_{\mathcal{M}}} w_{\min}} \left( \frac{\tau}{8} \right)^{-d_{\cal M}}.
\]
we have, for any measure $\sigma_n$ of n points
\[
W_1(\sigma_n, \rho) \geq \frac{1}{32} \left( \frac{d_{\cal M} \, \mathrm{vol} \, {\cal M}}{4 \omega_{d_{\mathcal{M}}} w_{\min}
} \right)^{\frac{1}{d_{\cal M}}} n^{-\frac{1}{d_{\cal M}}}.
\]
Also, for a constant $C \leq 48$, we have 
\[
\mathbb{E} [W_1(\rho_n, \rho)] \le C \left( \frac{\operatorname{vol} {\cal M}}{n w_{\min}} \right)^{\frac{1}{d}} \sqrt{\log \left( \frac{n w_{\min} \mathrm{Diam}(\rho)^{d_{\cal M}}}{d_{\cal M} \operatorname{vol}_{\cal M}} \right)}.
\]
\end{proposition}

Using Proposition \ref{prop::from_block},  we thus can apply our Theorem \ref{th::dim_estim}, with  the constants: $\dint =d _{\cal M}$, $n_{\min} = \frac{\dint \, \mathrm{vol} \, {\cal M}}{4 w \omega_{\min}} \left( \frac{\tau}{8} \right)^{-\dint}$, $n_{\max}=+\infty$, $C_1 = \frac{1}{32} \left( \frac{\dint  \mathrm{vol} \, {\cal M}}{4 w \omega_{\min}} \right)$, and $C_2 =  C \left( \frac{\operatorname{vol} {\cal M}}{ \omega_{\min}} \right)^{\frac{1}{\dint}} \sqrt{\log \left( \frac{n \omega_{\min} \mathrm{Diam}(\rho) ^{\dint}}{\dint \operatorname{vol}_{\cal M}} \right)}$.

\section{Proof of Proposition \ref{prop::richardson}: Convergence of Diagonal Richardson}

\begin{proof}
We analyze the bias and stochastic deviation separately.

Let the regularization parameter scale as $\varepsilon_n \asymp n^{-a}$.  Based on standard expansions and our assumption that the statistical bound holds with equality, we have:
\begin{align}
    \text{Bias}_{\text{ent}}(n) &= C_{\text{ent}} n^{-2a} + o(n^{-2a})\ , \label{eq::bias_ent} \\
    \text{Bias}_{\text{stat}}(n) &= C_{\text{stat}} n^{-(1 - a\dint)/2} + o(n^{-(1 - a\dint)/2}) \ . \label{eq::bias_stat}
\end{align}
To achieve the optimal convergence rate, we balance the leading order terms by equating their exponents:
\[
    2a = \frac{1 - a\dint}{2} \iff 4a = 1 - a\dint \iff a(\dint+4) = 1.
\]
This yields the optimal decay choice $a = \frac{1}{\dint+4}$. Under this choice, both bias terms decay at the common rate $\gamma$:
\[
    \gamma := 2a = \frac{1 - a\dint}{2} = \frac{2}{d+4}.
\]
Consequently, the total expectation admits the expansion:
\begin{equation}
    \label{eq:total_bias_expansion}
    \mathbb{E}[\widehat{S}_{\varepsilon_n, n}] = W_2^2 + \underbrace{(C_{\text{ent}} + C_{\text{stat}})}_{C_{\text{total}}} n^{-\gamma} + o(n^{-\gamma}).
\end{equation}

The diagonal Richardson estimator is defined as $\widehat{R}^{\text{Diag}}_{2n} := w_{2n} \widehat{S}_{\varepsilon_{2n}, 2n} + w_n \widehat{S}_{\varepsilon_{n}, n}$. By linearity, its expected value is:
\[
    \mathbb{E}[\widehat{R}^{\text{Diag}}_{2n}] = w_{2n} \mathbb{E}[\widehat{S}_{\varepsilon_{2n}, 2n}] + w_n \mathbb{E}[\widehat{S}_{\varepsilon_{n}, n}].
\]
Substituting the expansion from \eqref{eq:total_bias_expansion} separately for the entropic and statistical terms:
\begin{align*}
    \mathbb{E}[\widehat{R}^{\text{Diag}}_{2n}] &= W_2^2(w_{2n} + w_n) \\
    &\quad + C_{\text{ent}} \left( w_{2n} (2n)^{-\gamma} + w_n n^{-\gamma} \right) \quad (\text{Entropic First Order}) \\
    &\quad + C_{\text{stat}} \left( w_{2n} (2n)^{-\gamma} + w_n n^{-\gamma} \right) \quad (\text{Statistical First Order}) \\
    &\quad + o(n^{-\gamma}).
\end{align*}
The weights are chosen as $w_{2n} = \frac{2^\gamma}{2^\gamma - 1}$ and $w_n = \frac{-1}{2^\gamma - 1}$. We verify the cancellation conditions:
\begin{enumerate}
    \item $w_{2n} + w_n = \frac{2^\gamma - 1}{2^\gamma - 1} = 1$.
    \item \ $w_{2n} (2n)^{-\gamma} + w_n = \frac{2^\gamma \cdot 2^{-\gamma} - 1}{2^\gamma - 1} = \frac{1 - 1}{2^\gamma - 1} = 0$.
\end{enumerate}
Because both biases scale with $n^{-\gamma}$, the \emph{same} weights eliminate the leading terms of both $C_{\text{ent}}$ and $C_{\text{stat}}$ simultaneously. Thus:
\begin{equation}
    \left| \mathbb{E}[\widehat{R}^{\text{Diag}}_{2n}] - W_2^2\right| = o(n^{-\gamma}).
\end{equation}

Using Proposition 4 from \cite{chizat2020faster}, we have $\PP\left[|\widehat{S}_{\varepsilon, n} - \EE[\widehat{S}_{\varepsilon, n} ]| \geq t  \right] \leq 2\exp(-nt^2/D^24)$, therefore, noting $Z_n = \widehat{S}_{\varepsilon_n, n} - \mathbb{E}[\widehat{S}_{\varepsilon_n, n}]$, we have 
\begin{align*}
    \mathbb{E}\left[ \left| \widehat{R}^{\text{Diag}}_{2n} - \mathbb{E}[\widehat{R}^{\text{Diag}}_{2n}] \right| \right] 
    &\leq |w_{2n}| \mathbb{E}[|Z_{2n}|] + |w_n| \mathbb{E}[|Z_n|] \\
    &\leq |w_{2n}| \sqrt{\frac{2\pi D^2}{n}} + |w_n| \sqrt{\frac{4\pi D^2}{n}} = O(n^{-1/2}).
\end{align*}

Combining the residual bias and stochastic deviation bounds, we conclude
\[
    \mathbb{E}\left[ |\widehat{R}^{\text{Diag}}_{2n} - W_2^2| \right] \leq o(n^{-\gamma}) + O(n^{-1/2}) = o(n^{-\frac{2}{\dint+4}}).
\]
\end{proof}

\textbf{Remark: Robustness to Rate Mismatch and Small Sample Efficiency.} 
While the proof above assumes perfect balancing ($\gamma_{\text{ent}} = \gamma_{\text{stat}} = \gamma$), the diagonal Richardson estimator remains highly effective even if the statistical bias decays at a slightly faster rate $\beta > \gamma$ (making it asymptotically negligible compared to entropic bias). This is particularly relevant in the small sample regime ($n$ small), where "asymptotically better" terms may still possess large constants.

Let the statistical bias be of the form $\text{Bias}_{\text{stat}}(n) = C_{\text{stat}} n^{-\beta}$. The statistical first order bias after extrapolation becomes:
\begin{align*}
 w_{2n} C_{\text{stat}} (2n)^{-\beta} + w_n C_{\text{stat}} n^{-\beta} 
    &= C_{\text{stat}} n^{-\beta} \left( w_{2n} 2^{-\beta} + w_n \right).
\end{align*}
Substituting the weights $w_{2n} = \frac{2^\gamma}{2^\gamma - 1}$ and $w_n = \frac{-1}{2^\gamma - 1}$, we define the \textit{Bias Reduction Factor} $\rho(\beta, \gamma)$:
\begin{equation}
    \rho(\beta, \gamma) := \frac{w_{2n} 2^{-\beta} + w_n}{1} = \frac{2^{\gamma - \beta} - 1}{2^\gamma - 1}.
\end{equation}
We observe two key properties:
\begin{enumerate}
    \item \textbf{Exact Cancellation:} If $\beta = \gamma$, then $\rho(\gamma, \gamma) = 0$, recovering the main theorem.
    \item \textbf{Local Continuity (Small $n$ Regime):} In optimal transport, the entropic rate $\gamma = 2a$ and statistical rate $\beta = (1-a\dint)/2$ are algebraically coupled. For any choice of $a$ near the optimal value $\frac{1}{\dint+4}$, the difference $\delta = \beta - \gamma$ is small. A Taylor expansion around $\delta \approx 0$ yields:
    \[
        \rho(\beta, \gamma) \approx \frac{-\delta \ln 2}{2^\gamma - 1}.
    \]
    Thus, even if the rates are not identical, the estimator suppresses the statistical bias by a factor proportional to the rate mismatch $\delta$ for computationally tractable $n$ (where $n^{-\beta}$ is not yet negligible).  This further shows that this method is also robust to our estimation of $\dint$. 
\end{enumerate}

In the same way, we observe that Richardson is robust to small oscillations of $C_{\text{stat}}$.

\textbf{Remark: different bias conjecture}
Recent results by \citet{goldfeld2024limit} establish a Central Limit Theorem for the Sinkhorn Divergence, noting that the expected finite-sample error scales as $\mathbb{E}[S_{\varepsilon,n}] - S_\varepsilon = {\cal O}(1/\sqrt n+ n^{-1}r_{\varepsilon}) $ (see their Remark 7). The factor $r_{\varepsilon}$ depends on the variance of the Sinkhorn potentials. If we conjecture the scaling $r_{\varepsilon} \asymp \varepsilon^{-(\dint+2)}$, coming from Sinkhorn second derivative, and  the statistical bias becomes ${\cal O}(n^{-1} \varepsilon^{-(\dint+2)})$. Balancing this against the entropic bias $O(\varepsilon^2)$ yields the optimal decay choice $\varepsilon \asymp n^{-1/(\dint+4)}$. Notably, this recovers the exact same exponent $a = \frac{1}{\dint+4}$ derived under the standard $\sqrt{n}$-rate assumption, suggesting that our diagonal Richardson strategy is structurally robust to different statistical regimes.

\section{Proof of Proposition \ref{prop::var_reduc}: Variance Reduction with Bagging}

\begin{proof}

Let $\mathbf{D}_{2N} = \{(X_i, Y_i)\}_{i=1}^{2N}$ be a dataset of $2N$ i.i.d.\ samples drawn from the joint distribution of $(\mu, \nu)$.  We generate $K$ bags, where each bag corresponds to a subsample index set $\mathcal{S}_k$ of size $|\mathcal{S}_k| = N$, drawn uniformly at random with replacement. 

We analyze the asymptotic variance using the Law of Total Variance, conditioning on the observed dataset $\mathbf{D}_{2N}$. The total variance decomposes as:
\begin{equation} \label{eq:total_variance}
    \mathrm{Var}(\widehat{R}_{K,N}) = \underbrace{\mathrm{Var}\left( \mathbb{E}[\widehat{R}_{K,N} \mid \mathbf{D}_{2N}] \right)}_{\text{Term I: Dataset Variance}} + \underbrace{\mathbb{E}\left[ \mathrm{Var}(\widehat{R}_{K,N} \mid \mathbf{D}_{2N}) \right]}_{\text{Term II: Bagging Variance}}.
\end{equation}

\textbf{Asymptotic Linear Expansion.}

We first use that the empirical Sinkhorn Divergence is second-order Hadamard differentiable \cite{goldfeld2024limit}. Therefore, we can use the second order Taylor expansion. The first-order term $L_\varepsilon$ is linear and can be rewritten such as 
\[
L_\varepsilon =  \frac{1}{2N}\sum_{i=1}^{2N} \phi_\varepsilon(Z_i) 
\]
where $\phi_\varepsilon$ is the so-called influence function.The first-order von Mises expansion gives:
\begin{equation} \label{eq:expansion}
    \widehat{S}_{\varepsilon,2N} = S_\varepsilon + \frac{1}{2N}\sum_{i=1}^{2N} \phi_\varepsilon(Z_i) + r_{2N}(\varepsilon),
\end{equation}
where $E$ is the population value, $Z_i = (X_i, Y_i)$, and $\phi$ is the canonical influence function with $\mathbb{E}[\phi(Z)] = 0$ and $\mathrm{Var}(\phi_\varepsilon(Z)) = \sigma_\varepsilon^2$. 

Similarly, for any subsample $\mathcal{S}_k$ of size $N$, the estimator expands as:
\begin{equation}
    \widehat{S}_{\varepsilon,N}^{(k)} = S_\varepsilon + \frac{1}{N}\sum_{j \in \mathcal{S}_k} \phi_\varepsilon(Z_j) + r_{N}^{(k)}(\varepsilon).
\end{equation}

\textbf{Analysis of Term I (Dataset Variance).}
We analyze the conditional expectation $\mathbb{E}[\widehat{R}_{K,N} \mid \mathbf{D}_{2N}]$ by substituting the von Mises expansions.
First, recall that for the subsample mean term, the linearity of expectation implies:
\begin{equation}
    \mathbb{E}\left[ \frac{1}{M}\sum_{j \in \mathcal{S}_k} \phi_\varepsilon(Z_j) \,\middle|\, \mathbf{D}_{2N} \right] = \sum_{i=1}^N \phi_\varepsilon(Z_i) \cdot \mathbb{P}(i \in \mathcal{S}_k) \cdot \frac{1}{M} = \frac{1}{N}\sum_{i=1}^N \phi_\varepsilon(Z_i).
\end{equation}
Now, substituting the explicit expansions $\widehat{S}_{\varepsilon,N} = S_\varepsilon + \bar{\phi_\varepsilon}_N + r_N(\varepsilon)$ (where $\bar{\phi_\varepsilon}_N = \frac{1}{N}\sum \phi_\varepsilon(Z_i)$) and $    \widehat{S}_{\varepsilon,N}^{(k)} = S_\varepsilon + \bar{\phi_\varepsilon}_{M,k} + r_{N}^{(k)}(\varepsilon)$ into the definition of the estimator:
\begin{align}
    \mathbb{E}[\widehat{R}_{K,N} \mid \mathbf{D}_{2N}] &= (1+\lambda)\widehat{S}_{\varepsilon,N} - \lambda \mathbb{E}[\widehat{S}_{\varepsilon,N}^{(1)}  \mid \mathbf{D}_{2N}] \\
    &= (1+\lambda)(E + \bar{\phi_\varepsilon}_N + r_N) - \lambda \left( E + \mathbb{E}[\bar{\phi_\varepsilon}_{M,1} \mid \mathbf{D}_{2N}] + \mathbb{E}[r_{M,1} \mid \mathbf{D}_{2N}] \right) \\
    &= (1+\lambda)(E + \bar{\phi_\varepsilon}_N + r_N) - \lambda \left( E + \bar{\phi_\varepsilon}_N + \mathbb{E}[r_{M,1} \mid \mathbf{D}_{2N}] \right).
\end{align}
Grouping the terms by order:
\begin{align}
    \mathbb{E}[\widehat{R}_{K,2N} \mid \mathbf{D}_{2N}] &= \underbrace{\left[(1+\lambda) - \lambda\right] E}_{\text{Bias}} + \underbrace{\left[(1+\lambda) - \lambda\right] \bar{\phi_\varepsilon}_{2N}}_{\text{Linear Term}} + \underbrace{\left( (1+\lambda)r_{2N} - \lambda \mathbb{E}[r_{N,1} \mid \mathbf{D}_{2N}] \right)}_{\text{Pooled Remainder } \rho_{2N}} \\
    &= E + \frac{1}{2N}\sum_{i=1}^{2N} \phi_\varepsilon(Z_i) + \rho_{2N}.
\end{align}
We seek the variance of this quantity. Let $V_N = \frac{1}{N}\sum \phi_\varepsilon(Z_i)$.
\begin{equation}
    \mathrm{Var}\left( \mathbb{E}[\widehat{R}_{K,2N} \mid \mathbf{D}_{2N}] \right) = \mathrm{Var}(V_{2N}) + \mathrm{Var}(\rho_{2N}) + 2\mathrm{Cov}(V_{2N}, \rho_{2N}).
\end{equation}

\textbf{Linear Variance:} Since $\phi_\varepsilon(Z_i)$ are i.i.d., $\mathrm{Var}(V_{2N}) = \frac{\sigma_\e^2}{N}$.

\textbf{Remainder Variance:} 
The term $r_N$ is the second-order error of the von Mises expansion, satisfying $r_N = O(\varepsilon^{-c\dint}N^{-1})$ due to second-order Hadamard differentiability \cite{goldfeld2024limit, stromme2024minimum}. By the compactness assumption, moments converge, implying $\mathrm{Var}(r_N) = O(\varepsilon^{-2c\dint}N^{-1}N^{-2})$. Similarly, $\mathrm{Var}(r_{M,1}) = O(\varepsilon^{-2c\dint}N^{-2})$. Thus, $\mathrm{Var}(r_N) = O(\varepsilon^{2c\dint}N^{-2})$.

3. \textbf{Covariance:} By Cauchy-Schwarz, $\mathrm{Cov}(V_N, \rho_N) \le \sqrt{\mathrm{Var}(V_N)\mathrm{Var}(\rho_N)}  = O(N^{-1.5}\e^{-c\dint})$.

Additionning both variance terms, we have
\begin{equation} \label{eq:term1}
    \mathrm{Var}\left( \mathbb{E}[\widehat{R}_{K,N} \mid \mathbf{D}_{2N}] \right) = \frac{\sigma_{\varepsilon}^2}{N} + O(N^{-1.5}\e^{-c\dint}) + O(\varepsilon^{-2c\dint}N^{-2}).
\end{equation}

\textbf{Analysis of Term II (Bagging Variance).}
Conditioned on $\mathbf{D}_{2N}$, the full sample estimator $\widehat{S}_{\varepsilon,N}$ is constant. The variance arises solely from the $K$ subsampled estimators. Since the $K$ bags are drawn independently (conditional on $\mathbf{D}_{2N}$):
\begin{equation}
    \mathrm{Var}(\widehat{R}_{K,N} \mid \mathbf{D}_{2N}) = \lambda^2 \mathrm{Var}\left( \frac{1}{K}\sum_{k=1}^K \widehat{E}_{M,\mathcal{S}_k} \,\middle|\, \mathbf{D}_{2N} \right) = \frac{\lambda^2}{K} \mathrm{Var}(\widehat{S}_{\varepsilon,N}^{(1)}  \mid \mathbf{D}_{2N}).
\end{equation}
The estimator on the subsample $\widehat{S}_{\varepsilon,N}^{(1)} $ behaves asymptotically as the sample mean of the influence functions drawn without replacement from the finite population $\mathbf{D}_{2N}$.
Let $S^2_{\phi_\varepsilon, 2N}$ be the exact sample variance of the influence function on $\mathbf{D}_{2N}$. Using the finite population correction formulata, for the variance of a sample mean under simple random sampling with replacement, the conditional variance of the linear term is:
\begin{equation}
    \mathrm{Var}\left( \frac{1}{N}\sum_{j \in \mathcal{S}_1} \phi_\varepsilon(Z_j) \,\middle|\, \mathbf{D}_{2N} \right) = S^2_{\phi_\varepsilon, 2N} \left( \frac{1}{N} - \frac{1}{2N} \right) +o\left(\frac{S^2_{\phi_\varepsilon, 2N}}{N}\right).
\end{equation}
We now take the expectation over the dataset $\mathbf{D}_{2N}$. Since the samples $Z_i$ are i.i.d., $S^2_{\phi_\varepsilon, 2N}$ is an unbiased estimator of the population variance of the influence function. Thus, $\mathbb{E}[S^2_{\phi_\varepsilon, 2N}] = \mathrm{Var}(\phi_\varepsilon(Z)) = \sigma_\e^2$ exactly.
Therefore, 
\begin{equation} \label{eq:term2}
    \mathbb{E}\left[ \mathrm{Var}(\widehat{R}_{K,N} \mid \mathbf{D}_{2N}) \right] = \frac{\lambda^2}{K} \sigma_\e^2 \left[ \frac{1}{N} \right] + o\left(\frac{S^2_{\phi_\varepsilon, 2N}}{N}\right)  = \frac{\lambda^2}{K} \frac{\sigma_\e^2}{N} + o\left(\frac{S^2_{\phi_\varepsilon, 2N}}{N}\right).
\end{equation}

\textbf{Conclusion.}
Using that, by boundedness of the cost function on bounded support measures, we have $\sigma_\e \lesssim 1$, independently of $\varepsilon$. 
finally, substituting \eqref{eq:term1} and \eqref{eq:term2} back into \eqref{eq:total_variance}, as soon as $\e^{c\dint} = o(N)$, we have:
\begin{align}
    \mathrm{Var}(\widehat{R}_{K,N}) &= \left[ \frac{\sigma_{\varepsilon}^2}{N} \right] + \left[ \frac{\lambda^2}{K} \frac{\sigma_{\varepsilon}^2}{N} \right] +  O(N^{-1.5}\e^{-c\dint}) + O(\varepsilon^{-2c\dint}N^{-2}) + o\left(\frac{S^2_{\phi_\varepsilon, 2N}}{N}\right) \\
    &= \frac{\sigma_{\varepsilon}^2}{N} \left( 1 + \frac{\lambda^2}{K} \right) +  o(1)\ .
\end{align}
This completes the proof.
\end{proof}

\end{document}